\documentclass[conference, 12pt,onecolumn,draftclsnofoot]{IEEEtran}
\pdfoutput=1
\usepackage{amsmath,amssymb,amsfonts}
\usepackage{algorithm}
\usepackage{algorithmicx}
\usepackage[noend]{algpseudocode}
\usepackage{graphicx}
\usepackage{textcomp}
\usepackage{xcolor}

\usepackage{mathtools}
\usepackage[caption=false, font=footnotesize]{subfig}

\newcommand{\bmf}{\mathbf f}
\newcommand{\bmg}{\mathbf g}

\newcommand{\bmu}{\mathbf u}
\newcommand{\bmv}{\mathbf v}
\newcommand{\bmw}{\mathbf w}
\newcommand{\bmx}{\mathbf x}
\newcommand{\bmy}{\mathbf y}
\newcommand{\bmz}{\mathbf z}
\newcommand{\bmtheta}{\boldsymbol \theta}

\usepackage{bbm}

\newcommand{\mbF}{\mathbbm F}

\newcommand{\mbR}{\mathbbm R}

\newcommand{\mbX}{\mathbbm X}
\newcommand{\mbY}{\mathbbm Y}

\newcommand{\mcN}{\mathcal N}

\newcommand{\sfT}{\mathsf T}

\usepackage[acronym]{glossaries}

\newacronym{dnn}{DNN}{Deep Neural Network}
\newacronym{gcn}{GCN}{Graph Convolutional Network}
\newacronym{gnn}{GNN}{Graph Neural Network}
\newacronym{kfac}{KFAC}{Kronecker-Factored Approximate Curvature}
\newacronym{ngd}{NGD}{Natural Gradient Descent}
\newacronym{nggd}{NGGD}{Natural Graph Gradient Descent}
\newacronym{nn}{NN}{Neural Network}
\newacronym{sgd}{SGD}{Stochastic Gradient Descent}

\usepackage{amsthm}

\newtheorem{lemma}{Lemma}
\usepackage{biblatex}
\begin{document}

\title{Optimization of Graph Neural Networks with Natural Gradient Descent}

\makeatletter
\newcommand{\linebreakand}{%
  \end{@IEEEauthorhalign}
  \hfill\mbox{}\par
  \mbox{}\hfill\begin{@IEEEauthorhalign}
}
\makeatother

\author{
\IEEEauthorblockN{Mohammad Rasool Izadi \IEEEauthorrefmark{1}\IEEEauthorrefmark{2}} 
\IEEEauthorblockA{mizadi@nd.edu}
\and
\IEEEauthorblockN{Yihao Fang \IEEEauthorrefmark{2}} 
\IEEEauthorblockA{yfang5@nd.edu}
\and
\IEEEauthorblockN{Robert Stevenson \IEEEauthorrefmark{1}} 
\IEEEauthorblockA{rls@nd.edu}
\and
\IEEEauthorblockN{Lizhen Lin \IEEEauthorrefmark{2}}
\IEEEauthorblockA{lizhen.lin@nd.edu}
\linebreakand
\IEEEauthorrefmark{1}Electrical Engineering, \IEEEauthorrefmark{2}Applied and Computational Mathematics and Statistics\\
\textit{University of Notre Dame} \\
Notre Dame, IN, USA
}

\maketitle
\thispagestyle{plain}
\pagestyle{plain}

\begin{abstract}
In this work, we propose to employ information-geometric tools to optimize a graph neural network architecture such as the graph convolutional networks. 
More specifically, we develop optimization algorithms for the graph-based semi-supervised learning by employing the natural gradient information in the optimization process. 
This allows us to efficiently exploit the geometry of the underlying statistical model or parameter space for optimization and inference. 
To the best of our knowledge, this is the first work that has utilized the natural gradient for the optimization of graph neural networks that can be extended to other semi-supervised problems. 
Efficient computations algorithms are developed and extensive numerical studies are conducted to demonstrate the superior performance of our algorithms over existing algorithms such as ADAM and SGD.
\end{abstract}

\begin{IEEEkeywords}
Graph neural network, Fisher information, natural gradient descent, network data.
\end{IEEEkeywords}

\section{Introduction} \label{sec:intro}
In machine learning, the cost function is mostly evaluated using labeled samples that are not easy to collect.
Semi-supervised learning tries to find a better model by using unlabeled samples.
Most of the semi-supervised methods are based on a graph representation on (transformed) samples and labels \cite{kolesnikov2019revisiting}.
For example, augmentation methods create a new graph in which original and augmented samples are connected.
Graphs, as datasets with linked samples, have been the center of attention in semi-supervised learning.
\acrfull{gnn}, initially proposed to capture graph representations in neural networks \cite{scarselli2008graph}, have been used for semi-supervised learning in a variety of problems like node classification, link predictions, and so on.
The goal of each \acrshort{gnn} layer is to transform features while considering the graph structure by aggregating information from connected or neighboring nodes.
When there is only one graph, the goal of node classification becomes predicting node labels in a graph while only a portion of node labels are available (even though the model might have access to the features of all nodes).
Inspired by the advance of convolutional neural networks \cite{lecun1998gradient} in computer vision \cite{krizhevsky2012imagenet}, \acrfull{gcn} \cite{kipf2016semi}  employs the spectra of graph Laplacian for filtering signals and the kernel can be approximated using Chebyshev  polynomials or functions \cite{zhou2018graph, wu2020comprehensive}. 
GCN has become a standard and popular tool in the emerging field of geometric deep learning \cite{bronstein2017geometric}.

From the optimization perspective, \acrfull{sgd}-based methods that use an estimation of gradients have been popular choices due to their simplicity and efficiency. 
However, \acrshort{sgd}-based algorithms may be slow in convergence and hard to tune on large datasets.
Adding extra information about gradients, may help with the convergence but are not always possible or easy to obtain. 
For example, using second-order gradients like the Hessian matrix, resulting in the Newton method, is among the best choices which, however, are not easy to calculate especially in \acrshort{nn}s.
When the dataset is large or samples are redundant, \acrshort{nn}s are trained using methods built on top of \acrshort{sgd} like AdaGrad \cite{duchi2011adaptive} or Adam \cite{kingma2014adam}.
Such methods use the gradients information from previous iterations or simply add more parameters like momentum to the \acrshort{sgd}.
\acrfull{ngd} \cite{amari1998natural} provides an alternative based on the second-moment of gradients. Using an estimation of the inverse of the \textit{Fisher information matrix} (simply Fisher), \acrshort{ngd} transforms gradients into so-called \textit{natural gradients} that showed to be much faster compared to the \acrshort{sgd} in many cases.  The use of \acrshort{ngd} allows efficient exploration of the geometry of the underlying parameter space in the optimization process. 
Also, Fisher information plays a pivotal role throughout statistical modeling \cite{ly2017tutorial}.  
In frequentist statistics, Fisher information is used to construct hypothesis tests and confidence intervals by maximum likelihood estimators. 
In Bayesian statistics, it defines the Jeffreys's prior, a default prior commonly used for estimation problems and nuisance parameters in a Bayesian hypothesis test. 
In minimum description length, Fisher information measures the model complexity and its role in model selection within the minimum description length framework like AIC and BIC. 
Under this interpretation, NGD is invariant to any smooth and invertible reparameterization of the model, while \acrshort{sgd}-based methods highly depend on the parameterization.
For models with a large number of parameters like \acrshort{dnn}, Fisher is so huge that makes it almost impossible to evaluate natural gradients.
Thus, for faster calculation it is preferred to use an approximation of Fisher like \acrfull{kfac} \cite{martens2015optimizing} that are easier to store and inverse.

Both \acrshort{gnn} and training \acrshort{nn}s with \acrshort{ngd} have been active areas of research in recent years but, to the best of our knowledge, this is the first attempt on using \acrshort{ngd} in the semi-supervised learning.
In this work, a new framework for optimizing \acrshortpl{gnn} is proposed that takes into account the unlabeled samples in the approximation of Fisher. 
Section \ref{sec:bg} provides an overview of related topics such as semi-supervised learning, \acrshort{gnn}, and \acrshort{ngd}. 
The proposed algorithm is described in section \ref{sec:method} and a series of experiments are performed in section \ref{sec:exp} to evaluate the method's efficiency and sensitivity to hyper-parameters.
Finally, the work is concluded in section \ref{sec:conc}.

\section{Problem and Background} \label{sec:bg}
In this section, first, the graph-based semi-supervised learning with a focus on least-squared regression and cross-entropy classification is defined.
Required backgrounds on the optimization and neural networks are provided in the subsequent sections.
A detailed description of the notation is summarized in the Table~\ref{tab:notation}.

\begin{table}[htbp]
\caption{Notation}
\begin{center}
\begin{tabular}{c l}
Symbol & Description \\
\hline
$x, \bmx, X$ & Scalar, vector, matrix \\
$\epsilon, \lambda, \gamma$ & Regularization hyper-parameters \\
$\eta$ & The learning rate \\ 
$A$ & Adjacency matrix \\
$X^{\sfT}$ & Matrix transpose \\
$I$ & Comfortable identity matrix \\ 
$\underline{\bmx}$ & A sequence of $\bmx$ vectors \\
$n$ & The total number of samples \\
$\bar{n}$ & The number of labeled samples \\
$F$ & Fisher information matrix \\
$B$ & Preconditioning matrix \\
$r(\bmtheta)$ & The cost of parameters $\bmtheta$ \\
$l(\bmy, \hat{\bmy})$ & The loss between $\bmy$ and $\hat{\bmy}$ \\
$q(\bmx)$ & The source distribution \\
$q(\bmy|\bmx)$ & The target distribution \\
$q(a|\bmx, \bmx')$ & The adjacency distribution \\
$p(\bmy|f(X, A; \bmtheta))$ & The prediction distribution \\
$\phi(\cdot)$ & An element-wise nonliear function \\
$\nabla_{\bmtheta}f$ & Gradient of scalar $f$ wrt. $\bmtheta$ \\
$J_{\bmtheta} \bmf$ & Jacobian of vector $\bmf$ wrt. $\bmtheta$ \\
$H_{\bmtheta}f$ & Hessian of scalar $f$ wrt. $\bmtheta$ \\
$\odot$ & Element-wise multiplication operation \\
\end{tabular}
\label{tab:notation}
\end{center}
\end{table}

\subsection{Problem}
Consider an information source $q(\bmx)$ generating independent samples $\bmx_i \in \mbX$, the target distribution $q(\bmy|\bmx)$ associating $\bmy_i \in \mbY$ to each $\bmx_i$, and the adjacency distribution $q(a| \bmx, \bmx')$ representing the link between two nodes given their covariates levels $\bmx$ and  $\bmx'$. 
The problem of learning $q(\bmy|\bmx)$ is to estimate some parameters $\bmtheta$ that minimizes the cost function 
\begin{equation} \label{eq:loss1}
    r(\bmtheta) = E_{\bmx, \underline{\bmx}' \sim q(\bmx), \underline{a} \sim q(a|\bmx,\bmx'), \bmy \sim q(\bmy|\bmx)}[l(\bmy, \bmf(\bmx, \underline{\bmx}', \underline{a}; \bmtheta))]
\end{equation}
where the loss function $l(\bmy, \hat{\bmy})$ measures the prediction error between $\bmy$ and $\hat{\bmy}$.
Also, $\underline{\bmx'}$ and $\underline{a}$ show sequences of $\bmx'$ and $a$, respectively.
As $q(\bmx)$, $q(a| \bmx, \bmx')$, and $q(\bmy|\bmx)$ are usually unknown or unavailable, the cost $r(\bmtheta)$ is estimated using samples from these distributions.
Furthermore, it is often more expensive to sample from $q(\bmy|\bmx)$ than $q(\bmx)$ and $q(a| \bmx, \bmx')$ resulting in the different number of samples from each distribution being available.

Let $X=X_0$ to be a $d_0 \times n$ matrix of $n \ge 1$ i.i.d $\bmx_i$ samples from $q(\bmx)$ (equivalent to $X \sim q(X)$).
It is assumed that $n \times n$ adjacency matrix $A=[a_{ij}]$ is sampled from $q(a| \bmx_i, \bmx_j)$ for $i, j = 1, \dots, n$ (equivalent to $A \sim q(A|X)$).
One can consider $(X, A)$ to be a graph of $n$ nodes in which the $i$th column of $X$ shows the covariate at the node $i$ and $D=\text{diag}(\sum_j a_{ij})$ denotes the diagonal degree matrix.
Also, denote $Y$ to be a $d_m \times \bar{n}$ matrix of $\bar{n} < n$ samples $\bmy_i$ from $q(\bmy|\bmx_i)$ for $i=1, \dots, \bar{n}$ and $\bmz = [\mathbbm{1}(i \in \{ 1, \dots, \bar{n} \})]_{i=1}^n$ to be the training mask vector.
Note that $\mathbbm{1}(\text{condition})$ is $1$ if the condition is true and $0$ otherwise.
Thus, an empirical cost can be estimated by 
\begin{equation}
    \hat{r}(\bmtheta) = \frac{1}{\bar{n}} \sum_{i=1}^{\bar{n}} l(\bmy_i, \bmf(\bmx_i, X, A; \bmtheta)),
\end{equation}
where $\bmf(\bmx_i, X, A; \bmtheta)$ shows the processed $\bmx_i$ when having access to $n-1$ extra samples and links between them.
Note that as $X$ contains $\bmx_i$ (the $i$th column), $\bmf(\bmx_i, X, A; \bmtheta)$ and $\bmf(X, A; \bmtheta)$ are used interchangeably. 

Assuming $p(\bmy|\bmf(X, A; \bmtheta))=p_{\bmtheta}(\bmy|X, A)$ to be an exponential family with natural parameters in $\mbF$, the loss function becomes
\begin{equation} \label{eq:loss}
    l(\bmy, \bmf(X, A; \bmtheta)) = - \log p(\bmy|\bmf(X, A;\bmtheta)).
\end{equation}
In the least-squared regression, 
\begin{equation}
    p(\bmy|\bmf(X, A; \bmtheta)) = \mcN(\bmy|\bmf(X, A; \bmtheta), \sigma^2)
\end{equation}
for fixed $\sigma^2$ and $\mbF = \mbY = \mbR$. 
In the cross-entropy classification to $c$ classes, 
\begin{equation}\label{eq:softmax}
    p(y=k|\bmf(X, A; \bmtheta)) = \exp({\bmf_k}) / \sum_{j=1}^c \exp({\bmf_j})
\end{equation}
for $\mbF = \mbR^c$ and $\mbY = \{1, \dots, c\}$. 

\subsection{Parameter estimation}
Having the first order approximation of $r(\bmtheta)$ around a point $\bmtheta_0$, 
\begin{equation}
    r(\bmtheta) \approx r(\bmtheta_0) + \bmg_0^{\sfT} (\bmtheta - \bmtheta_0),
\end{equation}
the gradient descent can be used to update parameter $\bmtheta$ iteratively:
\begin{equation} \label{eq:ufo}
    \bmtheta_{t+1} = \bmtheta_t - \eta B \bmg_0
\end{equation}
where $\eta > 0$ denotes the learning rate, $\bmg_0 = \bmg(\bmtheta_0)$ is the gradient at $\bmtheta_0$ for
\begin{equation}
    \bmg(\bmtheta) = \frac{\partial r(\bmtheta)}{\partial \bmtheta}
\end{equation}
and $B$ shows a symmetric positive definite matrix called \textit{preconditioner} capturing the interplay between the elements of $\bmtheta$.
In SGD, $B = I$ and $\bmg_0$ is approximated by:
\begin{equation}
    \hat{\bmg}_0 = \frac{1}{\bar{n}} \sum_{i=1}^{\bar{n}} \frac{\partial l(y_i, \bmf(X, A; \bmtheta))}{\partial \bmtheta}
\end{equation}
where $\bar{n} \ge 1$ can be the mini-batch (a randomly drawn subset of the dataset) size.

To take into the account the relation between $\bmtheta$ elements, one can use the second order approximation of $r(\bmtheta)$:
\begin{equation}
    r(\bmtheta) \approx r(\bmtheta_0) + \bmg_0^{\sfT} (\bmtheta - \bmtheta_0) + \frac{1}{2}(\bmtheta - \bmtheta_0)^{\sfT}H_0(\bmtheta - \bmtheta_0),
\end{equation}
where $H_0 = H(\bmtheta_0)$ denotes the Hessian matrix at $\bmtheta_0$ for
\begin{equation}
    H(\bmtheta) = \frac{\partial^2 r(\bmtheta)}{\partial \bmtheta^{\sfT}\bmtheta}.
\end{equation}
Thus, having the gradients of $r(\bmtheta)$ around $\bmtheta$ as:
\begin{equation}
    \bmg(\bmtheta) \approx \bmg_0 + H_0(\bmtheta - \bmtheta_0),
\end{equation}
the parameters can be updated using:
\begin{equation} \label{eq:uso}
    \bmtheta_{t+1} = (I - \eta B H_0)\bmtheta_t - \eta B(\bmg_0 - H_0\bmtheta_0).
\end{equation}
The convergence of Eq.~\ref{eq:uso} heavily depends on the selection of $\eta$ and the distribution of $I-\eta B H_0$ eigenvalues.
Note that update rules Eqs.~\ref{eq:ufo} and~\ref{eq:uso} coincides at $B=H_0^{-1}$ resulting the Newton's method.
As it is not always possible or desirable to obtain Hessian, several preconditioners are suggested to adapt the information geometry of the parameter space. 

In NGD, the preconditioner is defined to be the inverse of Fisher Information matrix:
\begin{align} \label{eq:ngd}
    F(\bmtheta) := & E_{\bmx,\bmy \sim p(\bmx, \bmy; \bmtheta)}[\nabla_{\bmtheta} \nabla_{\bmtheta}^{\sfT}] \\
    = &  E_{\bmx \sim q(\bmx),\bmy \sim p(\bmy|\bmx; \bmtheta)}[\nabla_{\bmtheta} \nabla_{\bmtheta}^{\sfT}]
\end{align}
where $p(\bmx, \bmy; \bmtheta) := q(\bmx)p(\bmy|\bmx;\bmtheta)$ and 
\begin{equation}
    \nabla_{\bmtheta} := - \nabla_{\bmtheta} \log p(\bmx, \bmy; \bmtheta).
\end{equation}

\subsection{Neural Networks}
A neural network is a mapping from the input space $\mbX$ to the output space $\mbF$ through a series of $m$ layers. 
Layer $k \in \{1, \dots, m\}$, projects $d_{k-1}$-dimensional input $\bmx_{k-1}$ to $d_{k}$-dimensional output $\bmx_{k}$ and can be expressed as:
\begin{equation}\label{eq:layer}
    \bmx_{k} = \phi_{k}(W_k \bmx_{k-1})
\end{equation}
where $\phi_k$ is an element-wise non-linear function and $W_k$ is the $d_k \times d_{k-1}$-dimensional weight matrix.
The bias is not explicitly mentioned as it could be the last column of $W_k$ when $\bmx_k$ has an extra unit element.
Let the $\bmtheta = [\bmtheta_1, \dots, \bmtheta_m]$ to be the parameters of an $m$-layer neural network formed by stacking $m$ vectors of dimension $d_{k}d_{k-1}$ for $k=1, \dots, m$ and $\text{dim}(\bmx) = d_0$ such that $\text{dim}(\bmtheta) = \sum_{k=1}^m d_{k} d_{k-1}$.
The parameters of the $k$'th layer, $\bmtheta_k = \text{vec}(W_k)$ for $\text{vec}(W_k) = [\bmw_1, \dots, \bmw_{d_k}]$, is also shaped by piling up rows of $W_k$. 
Their gradients, $\nabla_{\bmtheta_k}$, could be written as: 
\begin{equation}
    \nabla_{\bmtheta_k} = \frac{\partial l}{\partial \bmtheta_k} = \frac{\partial \bmx_k}{\partial \bmtheta_k}^{\sfT} \frac{\partial l}{\partial \bmx_k}
\end{equation}
for $d_k \times d_kd_{k-1}$-dimensional matrix $\partial \bmx_k/ \partial \bmtheta_k$ and $d_k$-dimensional vector $\partial l/\partial \bmx_k$.

\subsection{Graph Neural Networks}
The \acrfull{gnn} extends the \acrshort{nn} mapping to the data represented in graph domains \cite{scarselli2008graph}.
The basic idea is to use related samples when the adjacency information is available.
In other words, the input to the $k$'th layer, $\bmx_{k-1}$ is transformed into $\tilde{\bmx}_{k-1}$ that take into the account unlabeled samples using the adjacency such that $p(\bmx_{k-1}, A) = p(\tilde{\bmx}_{k-1})$.
Therefore, for each node $i=1, \dots, n$, the Eq.~\ref{eq:layer} can be written by a local transition function (or a single message passing step) as:
\begin{equation} \label{eq:layer2}
    \bmx_{k, i} = \bmf_{k}(\bmx_{k-1, i}, \underline{\bmx}_{k-1, i}, \bmx_{0, i}, \underline{\bmx}_{0, i}; W_k)
\end{equation}
where $\underline{\bmx}_{k, i}$ denotes all the information coming from nodes connected to the $i$th node at the $k$th layer.
The subscripts here are used to indicate both the layer and the node, i.e. $\bmx_{k,i}$ means the state embedding of node $i$ in the layer $k$.
Also, the local transition Eq.~\ref{eq:layer2}, parameterized by $W_k$, is shared by all nodes that includes the information of the graph structure, and $\bmx_{0,i} = \bmx_{i}$.

The \acrfull{gcn} is a one of the \acrshort{gnn} with the message passing operation as a linear approximation to spectral graph convolution, followed by a non-linear activation function as:
\begin{align} \label{eq:layer3}
    \bmx_{k, i} = & \bmf_{k}(\bmx_{k-1, i}, \underline{\bmx}_{k-1, i}; W_k) \\
    X_{k} = & \phi_k(W_kX_{k-1}\tilde{A}) \\
    = & \phi_k(W_k\tilde{X}_{k-1})
\end{align}
where $\phi_k$ is a element-wise nonlinear activation function such as $\text{RELU}(x) = \text{max}(x, 0)$, $W_k$ is a $d_{k}\times d_{k-1}$ parameter matrix that needs to be estimated.
$\tilde{A}$ denotes the normalized adjacency matrix defined by:
\begin{equation}
    \tilde{A} = (D+I)^{-1/2}(A+I)(D+I)^{-1/2}
\end{equation}
to overcome the overfitting issue due to the small number of labeled samples $\bar{n}$.
In fact, a \acrshort{gcn} layer is basically a \acrshort{nn} (Eq.~\ref{eq:layer}) where the input $\bmx_{k-1}$ is initially updated into $\tilde{\bmx}_{k-1}$ using a so-called renormalization trick such that $\tilde{\bmx}_{k-1, i} = \sum_{j=1}^n \tilde{a}_{i,j} \bmx_{k-1, i}$ where $\tilde{A} = [\tilde{a}_{i,j}]$.
Comparing Eq.~\ref{eq:layer3} with the more general Eq.~\ref{eq:layer2}, the local transition function $\bmf_k$ is defined as a linear combination followed by a nonlinear activation function.
For classifying $\bmx$ into $c$ classes, having a $c$-dimensional $\bmx_m$ as the output of the last layer with a Softmax activation function, the loss between the label $\bmy$ and the prediction $\bmx_m$ becomes:
\begin{equation}
    l(\bmy, \bmx_m) = - \sum_{j=1}^c \mathbbm{1}(\bmx_{m, j} = j) \log \bmx_{m, j}.
\end{equation}

\section{Method} \label{sec:method}
The basic idea of preconditioning is to capture the relation between the gradients of parameters $\nabla_{\bmtheta}$.
This relation can be as complete as a matrix $B$ (for example, \acrshort{ngd}) representing the pairwise relation between elements of $\nabla_{\bmtheta}$ or as simple as a weighting vector (for example, Adam) with the same size as $\nabla_{\bmtheta}$ resulting in a diagonal $B$.
Considering the flow of gradients $\nabla_{\bmtheta, t}$ over the training time as input features, the goal of preconditioning is to extract useful features that help with the updating rule.
One can consider the preconditioner to be the expected value of $B(\bmx, \bmy) = [b_{ij}]^{-1}$ for 
\begin{equation}
    b_{ij} = b_{i, j}(\bmx, \bmy) = b({\nabla_{\theta}}_i || {\nabla_{\theta}}_j) \footnote{Note that the adjacency matrix provides the relation between $\bmx$ samples where the preconditioning matrix includes the relation between the elements of $\nabla_{\bmtheta}$}.
\end{equation}
In methods with a diagonal preconditioner like Adam, $B(\bmx, \bmy) = \text{diag}(\nabla_{\bmtheta} \odot \nabla_{\bmtheta})$, the pair-wise relation between gradients is neglected. 
Preconditioners like Hessian inverse in Newton's method with the form of $b_{ij} = \partial {\nabla_{\bmtheta}}_i/\partial \theta_j$ are based on the second derivative that encodes the cost curvature in the parameter space. In NGD and similar methods, this curvature is approximated using the second moment of gradient $b_{ij} = {\nabla_{\bmtheta}}_i {\nabla_{\bmtheta}}_j$, as an approximation of Hessian, in some empirical cases (see \cite{kunstner2019limitations} for a detailed discussion).

In this section, a new preconditioning algorithm, motivated by natural gradient, is proposed for graph-based semi-supervised learning that improves the convergence of Adam and SGD with intuitive and insensitive hyper-parameters.
The natural gradient is a concept from information geometry and stands for the steepest descent direction in the Riemannian manifold of probability distributions \cite{amari1998natural}, where the distance in the distribution space is measured with a special \textit{Riemannian metric}.
This metric depends only on the properties of the distributions themselves and not their parameters, and in particular, it approximates the square root of the KL divergence within a small neighborhood \cite{martens2014new}. 
Instead of measuring the distance between the parameters $\bmtheta$ and $\bmtheta'$, the cost is measured by the KL divergence between their distributions $p(\bmtheta)$ and $p(\bmtheta')$.
Consequently, the steepest descent direction in the statistical manifold is the negative gradient preconditioned with the Fisher information matrix $F(\bmtheta)$.
The validation cost on three different datasets is shown in Fig.~\ref{fig:val} where preconditioning is applied to both Adam and SGD. 

\begin{figure*}[htbp]
    \centering
    \subfloat[Adam - CiteSeer\label{fig:val1-citeseer}]{
    \includegraphics[width=.3\textwidth]{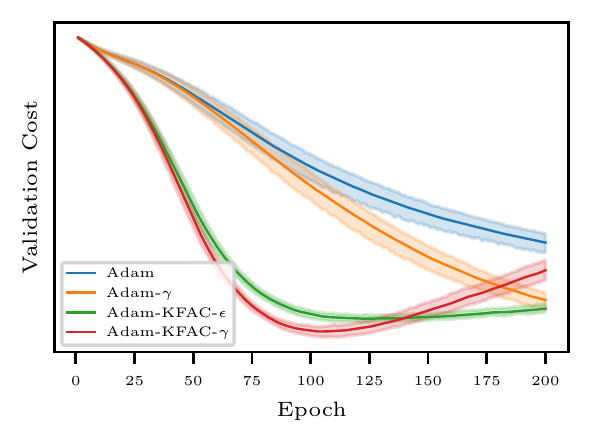}}
    \subfloat[Adam - Cora\label{fig:val1-cora}]{%
    \includegraphics[width=.3\textwidth]{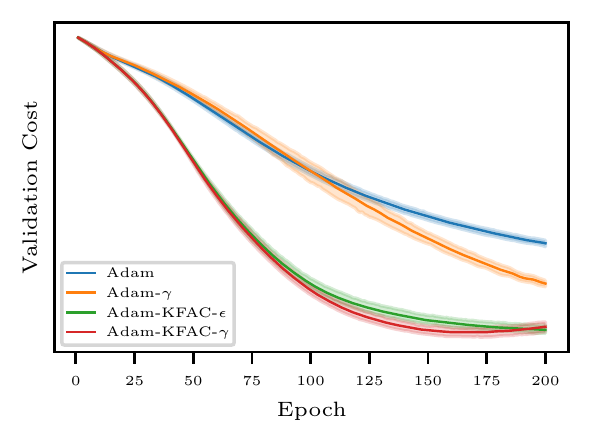}}
    \subfloat[Adam - PubMed\label{fig:val1-pubmed}]{%
    \includegraphics[width=.3\textwidth]{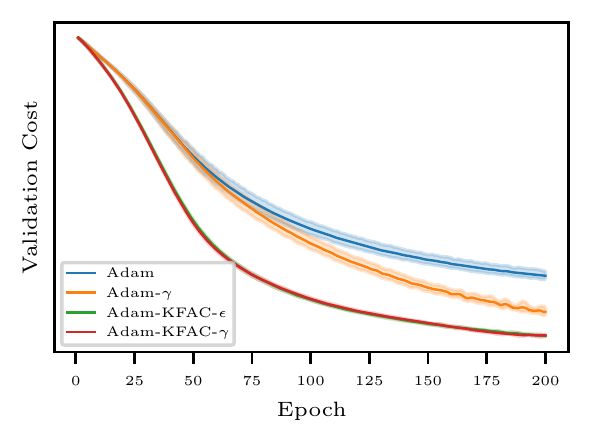}} \\
    \subfloat[SGD - CiteSeer\label{fig:val2-citeseer}]{
    \includegraphics[width=.3\textwidth]{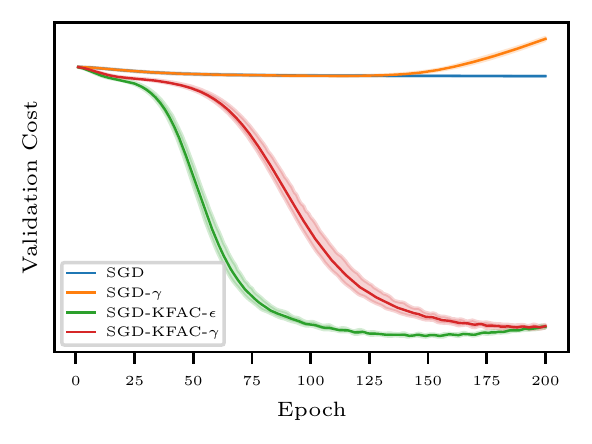}}
    \subfloat[SGD - Cora\label{fig:val2-cora}]{%
    \includegraphics[width=.3\textwidth]{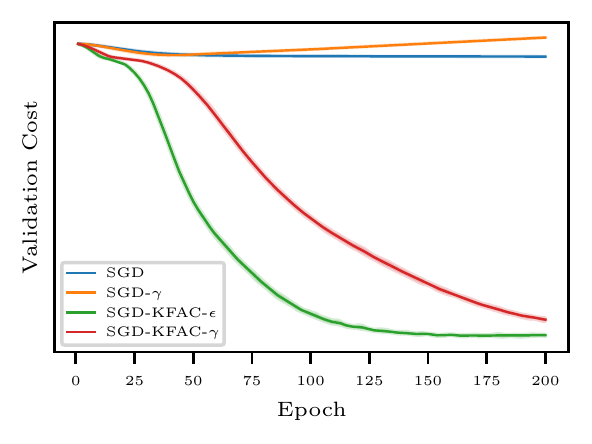}}
    \subfloat[SGD - PubMed\label{fig:val2-pubmed}]{%
    \includegraphics[width=.3\textwidth]{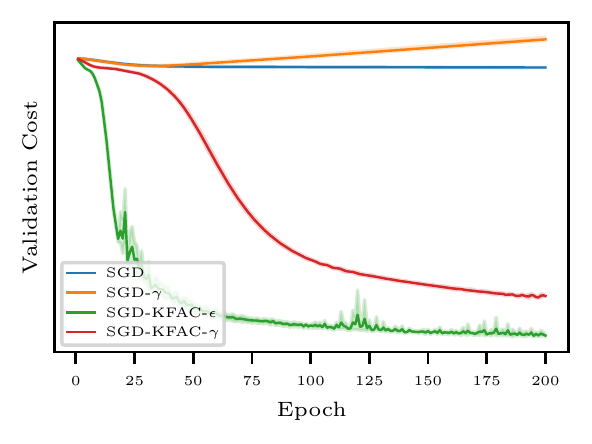}} \\
    \caption{The validation costs of four optimization methods on the second split of Citation datasets over $10$ runs. 
    A $2$-layer GCN with a $64$-dimensional hidden variable is used in all experiments. 
    As shown in Fig.~\ref{fig:val1-citeseer}, ~\ref{fig:val1-cora}, and ~\ref{fig:val1-pubmed} (upper row), the proposed Adam-KDAC methods (green and red curves) outperform vanilla Adam methods (blue and orange curves) on all three datasets. 
    Also, Fig.~\ref{fig:val2-citeseer}, ~\ref{fig:val2-cora}, and ~\ref{fig:val2-pubmed} (bottom row) reveal that the suggested SGD-KFAC methods (green and red curves) achieve a remarkably faster convergence than the vanilla SGD method (blue and orange curves) on all three datasets.}
  \label{fig:val} 
\end{figure*}

As the original NGD (Eq.~\ref{eq:ngd}) is defined based on a prediction function with access only to a single sample, $p(\bmy|\bmf(\bmx; \bmtheta))$, Fisher information matrix with the presence of the adjacency distribution becomes:
\begin{align} \label{eq:fisher2}
    F(\bmtheta) & = E_{\bmx, \underline{\bmx}', \underline{a}, \bmy \sim p(\bmx, \underline{\bmx}', \underline{a}, \bmy; \bmtheta)}[\nabla_{\bmtheta} \nabla_{\bmtheta}^{\sfT}] \\
    & = E_{\bmx, \underline{\bmx}' \sim q(\bmx), \underline{a} \sim q(a|\bmx, \bmx'), \bmy \sim p(\bmy|\bmx, \underline{\bmx}', \underline{a}; \bmtheta)}[\nabla_{\bmtheta} \nabla_{\bmtheta}^{\sfT}].
\end{align}
With $n$ samples of $q(\bmx)$, i.e. $X$ and $n^2$ samples of $q(a|X)$, i.e. $A$, Fisher can be estimated as: 
\begin{equation}\label{eq:fisher3}
    \hat{F}(\bmtheta) = E_{\bmy \sim p(\bmy|X, A; \bmtheta)}[\nabla_{\bmtheta} \nabla_{\bmtheta}^{\sfT}],
\end{equation} 
where 
\begin{equation}
    \nabla_{\bmtheta} = - \nabla_{\bmtheta} \log p(X, A, \bmy; \bmtheta).
\end{equation}
In fact, to evaluate the expectation in Eq.~\ref{eq:fisher2}, $q(X)$ and $q(A|X)$ are approximated with $\hat{q}(X)$ and $\hat{q}(A|X)$ using $\{\bmx_i\}_{i=1}^n$ and $A$, respectively.
However, there are only $\bar{n}$ samples from $\hat{q}(\bmy|\bmx_j)$ as an approximation of $q(\bmy|\bmx_j)$ for the following replacement: 
\begin{equation}
    p(\bmy|X, A;\bmtheta) \approx \hat{q}(\bmy|\bmx_i).
\end{equation}
Therefore, an empirical Fisher can be obtained by
\begin{equation} \label{eq:empfisher}
    \hat{F}(\bmtheta) = \frac{1}{\bar{n}} \sum_{i=1}^{\bar{n}} \nabla_{\bmtheta, i} \nabla_{\bmtheta, i}^{\sfT} = \sum_{i=1}^{\bar{n}} B_i(\bmtheta)
\end{equation}
for 
\begin{align}
    \nabla_{\bmtheta, i} & = - \nabla_{\bmtheta} \log p(\bmy_i|X, A; \bmtheta) \\
    B_i(\bmtheta) & = B(X, A, \bmy_i; \bmtheta).
\end{align}

From the computation perspective, the matrix $B_i(\bmtheta)$ can be very large, for example, in neural networks with multiple layers, the parameters could be huge, so it needs to be approximated too.
In networks characterized with Eqs.~\ref{eq:layer} or \ref{eq:layer3}, a simple solution would be ignoring the cross-layer terms so that $B_i(\bmtheta)^{-1}$ and consequently $B_i(\bmtheta)$ turns into a block-diagonal matrix:
\begin{equation}
    B_i(\bmtheta) = \text{diag}(B_{1, i}, \dots, B_{m, i})
\end{equation}
In \acrshort{kfac}, the diagonal block $B_{k, i}$, corresponded to $k$'th layer with the dimension $d_{k}d_{k-1} \times d_{k}d_{k-1}$, is approximated with the Kronecker product of the inverse of two smaller matrices $U_{k, i}$ and $V_{k, i}$ as:
\begin{equation}
    B_{k, i} = (U_{k, i} \otimes V_{k, i})^{-1} = U_{k, i}^{-1} \otimes V_{k, i}^{-1}.
\end{equation}
For $\nabla_{\bmtheta, i} = [\nabla_{\bmtheta_1, i}^{\sfT}, \dots, \nabla_{\bmtheta_m, i}^{\sfT}]^{\sfT}$, the preconditioned gradient $B_{k, i} \nabla_{\bmtheta_k, i}$ can be computed using the identity
\begin{align}
    B_{k, i} \nabla_{\bmtheta_k, i} & = U_{k, i}^{-1} \otimes V_{k, i}^{-1} \text{vec}(\frac{\partial l}{\partial W_k}) \\
    & = \text{vec}(U_{k, i}^{-1} \frac{\partial l}{\partial W_k} V_{k, i}^{-1}).
\end{align}
Noting that:
\begin{align}
    \frac{\partial l}{\partial W_k} &= \left(\frac{\partial l}{\partial \bmx_k} \odot \phi_k'(W_k \tilde{\bmx}_{k-1}) \right) \tilde{\bmx}_{k-1}^{\sfT} \\
     &=\bmu_{k, i} \bmv_{k, i}^{\sfT},
\end{align}
$U_{k}$ and $V_{k}$ blocks are approximated with the expected values of $\bmu_{k, i}\bmu_{k, i}^{\sfT}$ and $\bmv_{k, i}\bmv_{k, i}^{\sfT}$ respectively where $\text{dim}(\bmu_k) = d_k$, $\text{dim}(\bmv_k) = d_{k-1}$.
Finally, $U_{k}^{-1}$ and $V_{k}^{-1}$ are evaluated by taking inverses of $U_{k} + \epsilon^{-1/2}$ and $V_{k} + \epsilon^{-1/2}$ for $\epsilon$ being the regularization hyper-parameter.

For a graph with $n$ nodes, adjacency matrix $A$, and the training set $\{ (\bmx_i, \bmy_i) \}_{i=1}^{\bar{n}} + \{ \bmx_i\}_{i=\bar{n}+1}^n$, $U_{k}$ and $V_{k}$ are estimated in two ways: (1) using only $\bar{n}$ labeled samples, and (2) including $n-\bar{n}$ unlabeled samples. 
In the first method, $U_k$ and $V_k$ are estimated by:
\begin{gather}
    U_k = \frac{1}{\bar{n}} \left(\frac{\partial l}{\partial X_k} \odot \phi_k'(W_k \tilde{X}_{k-1}) \right) \left(\frac{\partial l}{\partial X_k} \odot \phi_k'(W_k \tilde{X}_{k-1}) \right)^{\sfT} \\
    V_k = \frac{1}{\bar{n}} \tilde{X}_{k-1} \tilde{X}_{k-1}^{\sfT}.
\end{gather}
Note that both $\partial l / \partial X_k$ and $\phi_k'(W_k \tilde{X}_{k-1})$ are $d_{k} \times n$ matrices and the last $n-\bar{n}$ columns of $\partial l / \partial X_k$ are zero. 
However, as unlabeled samples are not used in the first method, one needs to evaluate loss function for $i=\bar{n}+1, \dots, n$, which  can be done by sampling $\hat{\bmy}_i$ from $p(\bmy|\bmx;\bmtheta)$.
In the second method, these new samples are added to the empirical cost as
\begin{align}
    \hat{r}(\bmtheta) = & \frac{1}{\bar{n}} \sum_{i=1}^{\bar{n}} l(\bmy_i, f(X, A; \bmtheta)) \nonumber \\
    & + \frac{\lambda}{n-\bar{n}} \sum_{i=\bar{n}+1}^{n} l(\bmy_i, f(X, A; \bmtheta)),
\end{align}
where $0 \leq \lambda \leq 1$ denotes the regularization hyper-parameter for controlling the cost of predicted labels and $\lambda=0$ results the first method.
As the prediction improves over the course of training, $\lambda$ can be a function of iteration $t$, for example here, it is defined to be: 
\begin{equation}\label{eq:lambda}
    \lambda(t) := \left(\frac{t}{\text{max}(t)}\right)^{\gamma},
\end{equation}
where $\text{max}(t)$ shows the maximum number of iterations and $\gamma$ is the replaced regularization hyper-parameter.
Algorithm~\ref{alg:alg} shows the preconditioning step for modifying gradients of each layer at any iteration such that gradients are first, transformed using two matrices of $V_k^{-1}$ and $U_k^{-1}$, then sent to the optimization algorithm for updating parameters.

\begin{algorithm}[htbp]
\caption{Semi-Supervised Preconditioning}\label{alg:alg}
\begin{algorithmic}[0]
    \Require $\nabla W_k$ \Comment{Gradient of parameters for $k=1, \dots, m$}
    \Require $A$ \Comment{Adjacency matrix}
    \Require $D$ \Comment{Degree matrix}
    \Require $\bmz$ \Comment{Training mask vector}
    \Require $\epsilon$, $\lambda$ \Comment{Regularization hyper-parameters}
    \Statex
    \State $n = \text{dim}(\bmz)$
    \State $\bar{n} = \sum(\bmz)$
    \State $\tilde{A} = (D+I)^{-1/2}(A+I)(D+I)^{-1/2} = [\tilde{a}_{ij}]$
    \For{$k=1, \dots, m$}
    \State $\tilde{\bmx}_{k-1, i} = \sum_{j=1}^{n} \tilde{a}_{i, j} \bmx_{k-1, j}$
    \State $\bmu_{k-1, i} = \partial l /\partial \bmx_k \odot \phi_k'(W_k \tilde{\bmx}_{k-1, i})$
    \State $\bmv_{k-1, i} = \tilde{\bmx}_{k-1, i}$
    \State $U_k = \sum_{i=1}^{n} (z_i + (1-z_i)\lambda) \bmu_{k-1, i}\bmu_{k-1, i}^{\sfT} / (n+\lambda \bar{n})$
    \State $V_k = \sum_{i=1}^{n} (z_i + (1-z_i)\lambda) \bmv_{k, i}\bmv_{k, i}^{\sfT} / (n+\lambda \bar{n})$
    \State $U_k^{-1} = \Call{Inverse}{U_k}$
    \State $V_k^{-1} = \Call{Inverse}{V_k}$
    \State \textbf{output} $V_k^{-1} \nabla W_k U_k^{-1}$ 
    \EndFor
    
    \Statex
    \Function{Inverse}{$X$}
    \State \textbf{output} $(X + \epsilon^{-1/2}I)^{-1}$
    \EndFunction
\end{algorithmic}
\end{algorithm}

\subsection{Relation between Fisher and Hessian}
The Hessian of the cost function:
\begin{equation}
    H_{\bmtheta}r(\bmtheta) = E_{X, A, \bmy \sim p(X, A, \bmy; \bmtheta)}[ H_{\bmtheta} l(\bmy, f(X, A; \bmtheta))]
\end{equation}
can also be approximated using $\hat{q}(X)$, $\hat{q}(A|X)$, and $\hat{q}(\bmy|\bmx_i)$ resulting the empirical Hessian to be 
\begin{equation}
    \hat{H}_{\bmtheta}r(\bmtheta) := \frac{1}{\bar{n}} \sum_{i=1}^{\bar{n}} H_{\bmtheta} l(\bmy_i, f(X, A; \bmtheta)),
\end{equation}
which is equivalent to the empirical Fisher Eq.~\ref{eq:empfisher} when $p(X, A, \bmy; \bmtheta)$ is estimated with $\hat{q}(X)\hat{q}(A|X)\hat{q}(\bmy|\bmx_i)$ for $i=1, \dots, \bar{n}$ (see Lemma~\ref{lemma:1} in the appendix).

\section{Experiments} \label{sec:exp}
In this section, the performance of the proposed algorithm is evaluated compared to Adam and SGD on several datasets for the task of node classification in single graphs. 
The task is assumed to be transductive when all the features are available for training but only a portion of labels are used in the training.
First, a detailed description of datasets and the model architecture are provided.
Then, the general optimization setup, commonly used for the node classification, is specified.
The last part includes the sensitivity to hyper-parameter and training time analysis in addition to validation cost convergence and the test accuracy.
All the experiments are conducted mainly using Pytorch \cite{paszke2019pytorch} and Pytorch Geometric \cite{fey2019fast}, two open-source Python libraries for automating differentiation and working with graph datasets.

\subsection{Datasets}
Three citation datasets with the statistics shown in Table~\ref{tab:datasets} are used in the experiments \cite{sen2008collective}.
Cora, CiteSeer, and PubMed are single graphs in which nodes and edges correspond to documents and citation links, respectively.
A sparse feature vector (document keywords) and a class label are associated with each node.
Several splits of these datasets are used in the node classification task.
The first split, $20$ instances are randomly selected for training, $500$ for validation, and $1000$ for the test; the rest of the labels are not used \cite{yang2016revisiting}.
In the second split, all nodes except $500+1000$ validation and test nodes are used for the training \cite{chen2018fastgcn}.
To evaluate the overfitting behavior, the third split exploits all labels for training excluding $500+500$ nodes for the validation and test \cite{levie2018cayleynets}. 

\begin{table}[htbp]
    \centering
    \caption{Citation network datasets statistics}
    \begin{tabular}{lcccc}
        Dataset & Nodes & Edges & Classes & Features\\
        \hline
        Citeseer & 3,327 & 4732 & 6 & 3,703\\
        Cora & 2,708 & 5,429 & 7 & 1,433\\
        Pubmed & 19,717 & 44,338 & 3 & 500
    \end{tabular}
    \label{tab:datasets}
\end{table}

\subsection{Architectures}
In the node classification using a \acrshort{nn} followed by Softmax function (Eq.~\ref{eq:softmax}), the class with maximum probability is chosen to be the predicted node label.
A $2$-layer GCN with a $64$-dimensional hidden variable is used for comparing different optimization methods. 
In the first layer, the activation function ReLU is followed by a dropout function with a rate of $0.5$. 
The loss function is evaluated as the negative log-likelihood of Softmax (Eq.~\ref{eq:softmax}) of the last layer.

\subsection{Optimization}
The weights of parameters are initialized like the original \acrshort{gcn} \cite{kipf2016semi} and input vectors are row-normalized accordingly \cite{glorot2010understanding}.
The model is trained for $200$ epochs without any early stopping and the learning rate of $0.01$.
The Adam and SGD are used with the weight decay of $5 \times 10^{-4}$ and the momentum of $0.9$, respectively.

\subsection{Results}
The optimization performance is measured by both the minimum validation cost and the test accuracy for the best validation cost.
The validation cost of training a $2$-layer GCN with a $64$-dimensional hidden variable is used for comparing optimization methods (Adam and SGD) with their preconditioned version (Adam-KFAC and SGD-KFAC).
For each method, unlabeled samples are used in the training process with a ratio controlled by $\gamma$.
Fig.~\ref{fig:val} shows the validation cost of four methods based on Adam (upper row) and SGD (bottom row) for all three Citation datasets. 
The test accuracy of a $2$-layer GCN trained using four different methods on three split are shown in Tab.~\ref{tab:test-s1}, ~\ref{tab:test-s2}, and ~\ref{tab:test-s3}.
Reported values of test accuracy in tables are averages and $95\%$ confidence intervals over $10$ runs for the best hyper-parameters tuned on the second split of the CiteSeer dataset. 
Note that the test accuracy may not always reflect the performance of the optimization method as the objective function (cross-entropy) is not the same as the prediction function (argmax).
However, in most cases, the proposed method achieves better accuracy compared to Adam (the first row in all tables).
As a fixed learning rate $0.01$ is used in all methods, SGD has a very slow convergence and does not provide competitive results.

\begin{table}[htbp]
    \centering
    \caption{The test accuracy of four optimization methods on the first split of Citation datasets over $10$ runs. A $2$-layer GCN with a $64$-dimensional hidden variable is used in all experiments.}
    \begin{tabular}{lcccc}
         & CiteSeer & Cora & Pubmed\\
        \hline
        $\text{Adam}$ & $ 71.66 \pm  0.61$ & $ 81.20 \pm  0.25$ & $ 79.72 \pm  0.30$\\
        $\text{Adam}_{\gamma}$ & $\mathbf{74.28 \pm  0.67}$ & $ 82.42 \pm  0.33$ & $\mathbf{80.06 \pm  0.34}$\\
        $\text{Adam-KFAC}_{\epsilon}$ & $ 71.94 \pm  0.53$ & $ 81.68 \pm  0.25$ & $ 79.48 \pm  0.28$\\
        $\text{Adam-KFAC}_{\gamma}$ & $ 70.24 \pm  0.66$ & $\mathbf{82.84 \pm  0.87}$ & $ 76.94 \pm  0.59$\\
        \hline
        $\text{SGD}$ & $ 20.38 \pm  8.92$ & $ 23.14 \pm  5.17$ & $ 45.76 \pm  3.04$\\
        $\text{SGD}_{\gamma}$ & $ 17.64 \pm  6.18$ & $ 17.26 \pm  8.41$ & $ 46.20 \pm  4.35$\\
        $\text{SGD-KFAC}_{\epsilon}$ & $ 71.82 \pm  0.48$ & $\mathbf{82.06 \pm  0.34}$ & $ 77.20 \pm  0.63$\\
        $\text{SGD-KFAC}_{\gamma}$ & $\mathbf{73.52 \pm  0.22}$ & $ 81.70 \pm  0.79$ & $\mathbf{79.20 \pm  0.29}$\\
    \end{tabular}
    \label{tab:test-s1}
\end{table}

\begin{table}[htbp]
    \centering
    \caption{The test accuracy of four optimization methods on the second split of Citation datasets over $10$ runs. A $2$-layer GCN with a $64$-dimensional hidden variable is used in all experiments.}
    \begin{tabular}{lcccc}
         & CiteSeer & Cora & Pubmed\\
        \hline
        $\text{Adam}$ & $ 78.68 \pm  0.83$ & $ 87.36 \pm  0.47$ & $ 87.78 \pm  0.14$\\
        $\text{Adam}_{\gamma}$ & $ 77.98 \pm  0.39$ & $ 87.28 \pm  0.34$ & $ 87.52 \pm  0.30$\\
        $\text{Adam-KFAC}_{\epsilon}$ & $\mathbf{79.50 \pm  0.15}$ & $\mathbf{87.60 \pm  0.20}$ & $ \mathbf{88.46 \pm  0.24}$\\
        $\text{Adam-KFAC}_{\gamma}$ & $ 79.42 \pm  0.32$ & $ 86.60 \pm  0.30$ & $ 87.88 \pm  0.16$\\
        \hline
        $\text{SGD}$ & $ 20.80 \pm  2.12$ & $ 31.90 \pm  0.00$ & $ 43.22 \pm  1.42$\\
        $\text{SGD}_{\gamma}$ & $ 20.96 \pm  5.22$ & $ 31.90 \pm  0.00$ & $ 40.82 \pm  0.33$\\
        $\text{SGD-KFAC}_{\epsilon}$ & $\mathbf{79.48 \pm  0.40}$ & $\mathbf{87.54 \pm  0.43}$ & $ \mathbf{89.08 \pm  0.18}$\\
        $\text{SGD-KFAC}_{\gamma}$ & $ 77.32 \pm  0.27$ & $ 87.42 \pm  0.24$ & $ 88.18 \pm  0.30$\\
    \end{tabular}
    \label{tab:test-s2}
\end{table}

\begin{table}[htbp]
    \centering
    \caption{The test accuracy of four optimization methods on the third split of Citation datasets over $10$ runs. A $2$-layer GCN with a $64$-dimensional hidden variable is used in all experiments.}
    \begin{tabular}{lcccc}
         & CiteSeer & Cora & Pubmed\\
        \hline
        $\text{Adam}$ & $ 79.80 \pm  0.66$ & $ 89.44 \pm  0.41$ & $ 87.16 \pm  0.71$\\
        $\text{Adam}_{\gamma}$ & $ 79.64 \pm  0.32$ & $ 89.60 \pm  0.91$ & $ 87.44 \pm  0.27$\\
        $\text{Adam-KFAC}_{\epsilon}$ & $\mathbf{80.52 \pm  0.14}$ & $\mathbf{90.16 \pm  0.59}$ & $ \mathbf{87.84 \pm  0.21}$\\
        $\text{Adam-KFAC}_{\gamma}$ & $ 80.52 \pm  0.22$ & $ 89.24 \pm  0.64$ & $ 87.36 \pm  0.37$\\
        \hline
        $\text{SGD}$ & $ 15.04 \pm  1.70$ & $ 32.80 \pm  0.00$ & $ 41.96 \pm  0.44$\\
        $\text{SGD}_{\gamma}$ & $ 16.12 \pm  5.30$ & $ 32.80 \pm  0.00$ & $ 41.20 \pm  0.00$\\
        $\text{SGD-KFAC}_{\epsilon}$ & $\mathbf{79.76 \pm  0.75}$ & $\mathbf{89.88 \pm  0.14}$ & $ \mathbf{89.36 \pm  0.57}$\\
        $\text{SGD-KFAC}_{\gamma}$ & $ 78.52 \pm  0.28$ & $ 88.72 \pm  0.38$ & $ 87.88 \pm  0.80$\\
    \end{tabular}
    \label{tab:test-s3}
\end{table}

The importance of hyper-parameters $\epsilon$, $\gamma$ are shown in Fig.~\ref{fig:hyperparams}.
Figures \ref{fig:adam-eps} and \ref{fig:sgd-eps} depict the sensitivity of Adam and SGD to the $\epsilon$ parameter, respectively.
As the inverse of $\epsilon$ directly affects the same factor as the learning rate $\eta$, the smaller the $\epsilon$, the faster the convergence. 
However, choosing very small $\epsilon$ results in larger confidence intervals which are not desirable.
The effect of $\gamma$ on Adam and SGD are depicted in figures~\ref{fig:adam-gamma} and~\ref{fig:sgd-gamma}, respectively. 
When using Adam, due to its faster convergence compared to SGD, smaller $\gamma$, i.e. using more predictions leads to much wider confidence intervals.
In other words, the training process dominated by more labels results in a more stable convergence with a smaller variance.
Thus, for a stable estimation, $\lambda$ or $\gamma$ must be tuned with respect to the optimization algorithm because of their sensitivity to the convergence rate.
Since the Fisher matrix does not change considerably at each iteration, an experiment is performed to explore the sensitivity of validation loss to the frequency of updating Fisher.
In Figures~\ref{fig:adam-update} and~\ref{fig:sgd-update}, the validation cost over time is evaluated for updating Fisher every $4, 8, \dots, 128$ iterations.
When Fisher is updated more frequently, its computation takes more time hence the training process lasts longer (having other hyper-parameters fixed). 
Increasing the update frequency does not affect the performance to some extent, however, it largely reduces the training time.
As updating Fisher every $50$ or $100$ iterations, does not affect the final validation cost to a great extent, to speed up the training process, Fisher is updated every $50$ epochs in all of the experiments.

\begin{figure*}[htbp]
    \centering
    \subfloat[Adam\label{fig:adam-eps}]{
    \includegraphics[width=.3\textwidth]{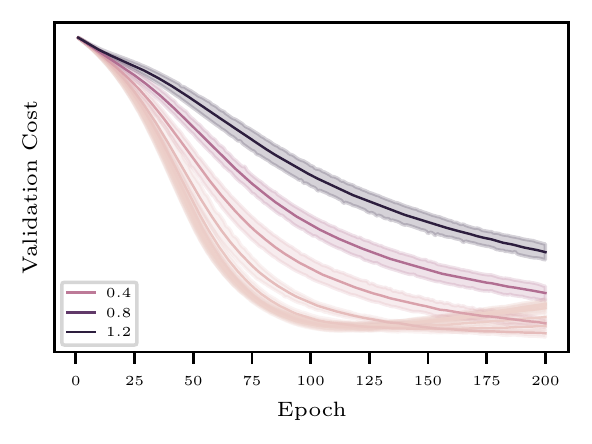}}
    \subfloat[Adam\label{fig:adam-gamma}]{
    \includegraphics[width=.3\textwidth]{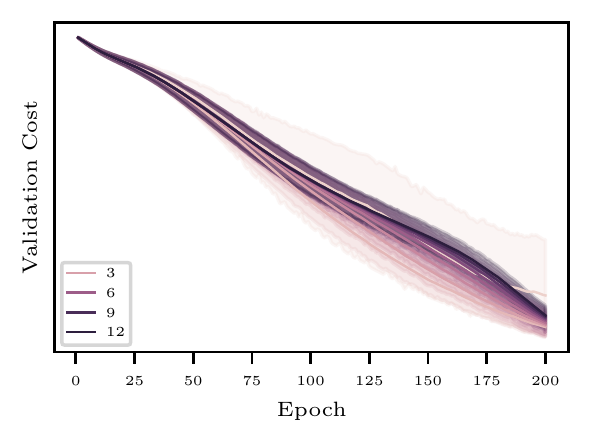}}
    \subfloat[Adam\label{fig:adam-update}]{
    \includegraphics[width=.3\textwidth]{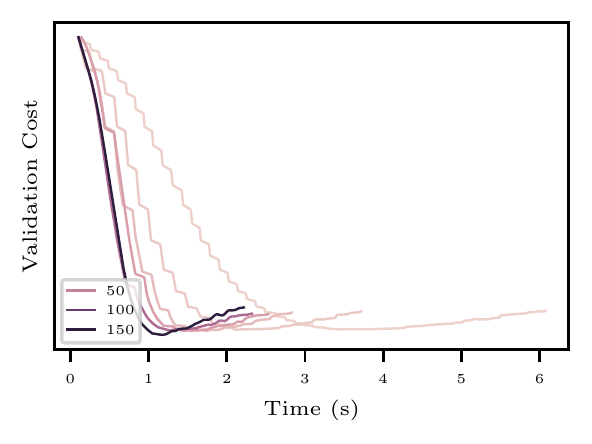}}\\
    \subfloat[SGD\label{fig:sgd-eps}]{%
    \includegraphics[width=.3\textwidth]{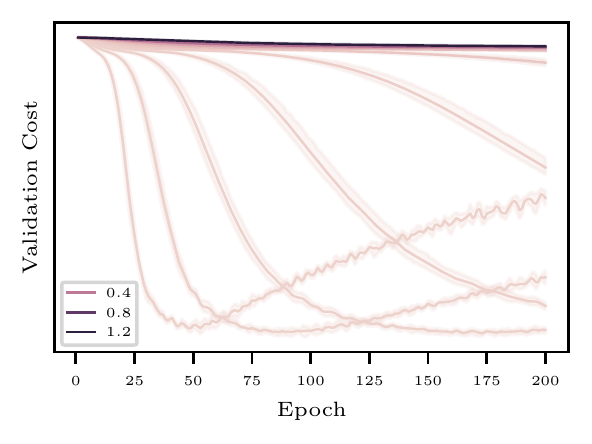}}
    \subfloat[SGD\label{fig:sgd-gamma}]{%
    \includegraphics[width=.3\textwidth]{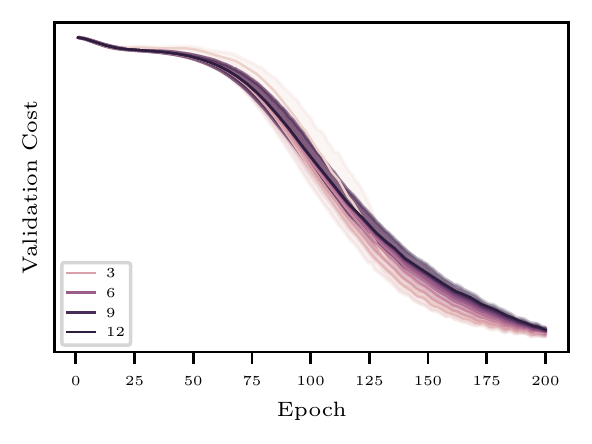}}
    \subfloat[SGD\label{fig:sgd-update}]{%
    \includegraphics[width=.3\textwidth]{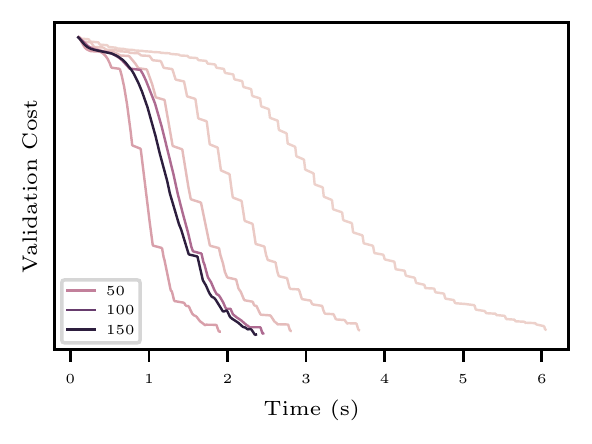}}
    \caption{The sensitivity of $\epsilon$, $\gamma$, and updating frequency on validation costs of Adam-KFAC (upper) and SGD-KFAC (below) when training on the second split of CiteSeer dataset over $10$ runs. 
    A $2$-layer GCN with a $64$-dimensional hidden variable is used in all experiments. 
    Fig.~\ref{fig:adam-eps} and~\ref{fig:sgd-eps} show that smaller $\epsilon$ results in a faster convergence with a probable cost of larger variance as it inversely scales the same factor as the learning rate.
    As depicted in Fig.~\ref{fig:adam-gamma} and~\ref{fig:sgd-gamma}, the larger the $\gamma$, the more stable the convergence (the more confined confidence intervals). 
    Finally, it can be seen in Fig.~\ref{fig:adam-update} and ~\ref{fig:sgd-update} that since performances are similar under different updating frequencies, selecting a relatively large frequency ($50$) can reduce the training time substantially.}
  \label{fig:hyperparams}
\end{figure*}

To examine the time complexity of the proposed method, the validation costs of Adam-KFAC and SGD-KFAC are compared with Adam and SGD when training on the second split of Citation datasets with respect to the training time for $200$ epochs (Fig.~\ref{fig:time}).
The training on Cora and PubMed (Fig.~\ref{fig:time-cora} and \ref{fig:time-pubmed}) takes a shorter time compared to the training on CitSeer (Fig.~\ref{fig:time-citeseer}) mainly because of the dimension of input features as it directly enlarges the size of the Fisher matrix. 
As shown in Fig.~\ref{fig:time}, the proposed SGD-KFAC method (red curve) converges much faster than the vanilla SGD as expected.
Surprisingly, SGD-KFAC outperforms Adam and even Adam-KFAC methods in all datasets implying that the naive SGD with a natural gradient preconditioner can lead to a faster convergence than Adam-based methods.
Another interesting observation is that Adam-based methods demonstrate similar performances in all experiments making them independent of the dataset while \acrshort{sgd}-based methods show different overfitting behavior.

\begin{figure*}[htbp]
    \centering
    \subfloat[CiteSeer\label{fig:time-citeseer}]{
    \includegraphics[width=.3\textwidth]{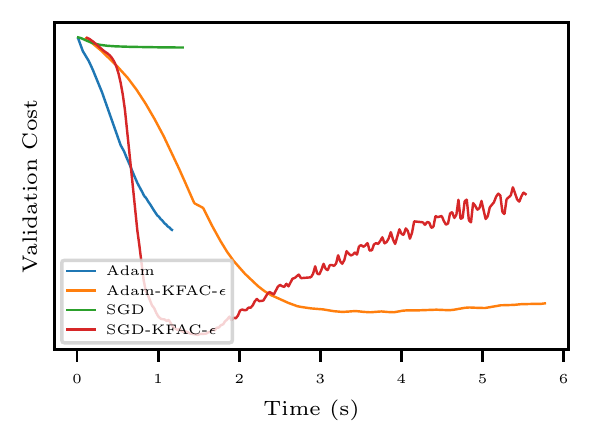}}
    \subfloat[Cora\label{fig:time-cora}]{%
    \includegraphics[width=.3\textwidth]{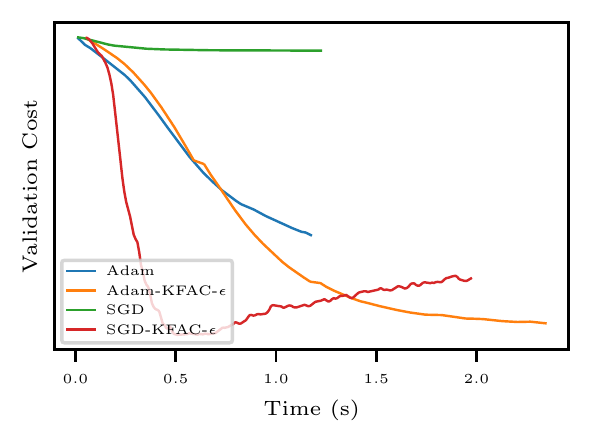}}
    \subfloat[PubMed\label{fig:time-pubmed}]{%
    \includegraphics[width=.3\textwidth]{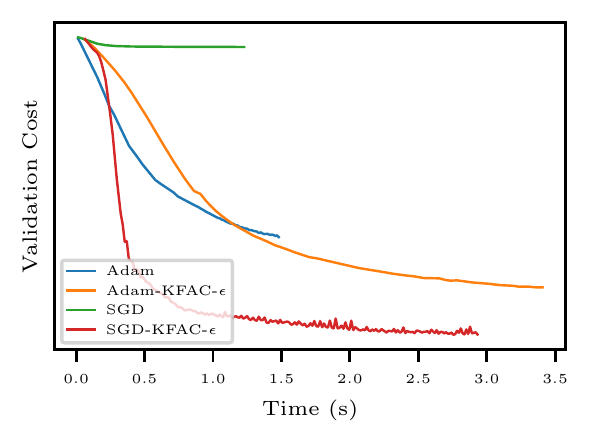}}
    \caption{The validation costs of four optimization methods with respect to the training time on the second split of Citation datasets over $10$ runs. 
    A $2$-layer GCN with a $64$-dimensional hidden variable is used in all experiments. 
    The proposed SGD-KFAC method shows the highest convergence rate among all other methods and it is slightly faster than Adam-KFAC.}
  \label{fig:time} 
\end{figure*}

\section{Conclusion} \label{sec:conc}
In this work, we introduced a novel optimization framework for graph-based semi-supervised learning.
After the distinct definition of semi-supervised problems with the adjacency distribution, we provided a comprehensive review of topics like semi-supervised learning, graph neural network, and preconditioning optimization (and \acrshort{ngd} as its especial case).
We adopted a commonly used probabilistic framework covering least-squared regression and cross-entropy classification.
In the node classification task, our proposed method showed to improve Adam and SGD not only in the validation cost but also in the test accuracy of \acrshort{gcn} on three splits of Citation datasets.
Extensive experiments were provided on the sensitivity to hyper-parameters and the time complexity.
As the first work, to the best of our knowledge, on the preconditioned optimization of graph neural networks, we not only achieved the best test accuracy but also empirically showed that it can be used with both Adam and SGD. 

As the preconditioner may significantly affect Adam, illustrating the relation between \acrshort{ngd} and Adam and effectively combining them can be a promising direction for future work. 
We also aim to deploy faster approximation methods than \acrshort{kfac} like \cite{george2018fast} and better sampling methods for exploiting unlabeled samples. Finally, since this work is mainly focused on single parameter layers, another possible research path would be adjusting KFAC to, for example, residual layers \cite{he2016deep}.
\section*{Acknowledgment}
YF and LL  were partially supported by NSF grants DMS Career 1654579 and DMS 1854779.

\printbibliography
\section*{Appendix}

\begin{lemma} \label{lemma:1}
The expected value of the Hessian of $-\log p(X, A, \bmy;\bmtheta)$ is equal to Fisher information matrix, i.e.
\begin{equation}
    -E_{X, A, \bmy \sim p(X, A, \bmy;\bmtheta)}[H_{\bmtheta}\log p(X, A, \bmy;\bmtheta)] = F
\end{equation}
\end{lemma}
\begin{proof}
The Hessian of $f(\bmtheta)$ can be written as the Jacobian of $\nabla_{\bmtheta}f$:
\begin{equation}
    H_{\bmtheta}f(\bmtheta) = J_{\bmtheta} \nabla_{\bmtheta} f(\bmtheta).
\end{equation}
So for the Hessian of the negative log-likelihood becomes:
\begin{align}
    - H_{\bmtheta} & \log p(X, A, \bmy;\bmtheta) \\
    = & - J_{\bmtheta} \frac{\nabla_{\bmtheta} p(X, A, \bmy;\bmtheta)}{p(X, A, \bmy;\bmtheta)} \\
    = & - \frac{H_{\bmtheta} p(X, A, \bmy;\bmtheta) . p(X, A, \bmy;\bmtheta)}{p(X, A, \bmy;\bmtheta) . p(X, A, \bmy;\bmtheta)} \\
      & - \frac{\nabla_{\bmtheta}p(X, A, \bmy;\bmtheta)\nabla_{\bmtheta}p(X, A, \bmy;\bmtheta)^{\sfT}}{p(X, A, \bmy;\bmtheta) . p(X, A, \bmy;\bmtheta)} \\
    & = - \frac{H_{\bmtheta} p(X, A, \bmy;\bmtheta)}{p(X, A, \bmy;\bmtheta)} + \nabla_{\bmtheta}\nabla_{\bmtheta}^{\sfT}
\end{align}
Taking the expectation over $p(X, A, \bmy;\bmtheta)$, the first term turns into zero:
\begin{align}
    E_{X, A, \bmy \sim p(X, A, \bmy;\bmtheta)} & [\frac{H_{\bmtheta} p(X, A, \bmy;\bmtheta)}{p(X, A, \bmy;\bmtheta)}] \\ 
    = & H_{\bmtheta} E_{X, A, \bmy \sim p(X, A, \bmy;\bmtheta)}[1] \\
    = & 0
\end{align}
and Fisher is defined as the expected value of the second term.
\end{proof}

\end{document}